\documentclass[runningheads,a4paper]{llncs}

\usepackage{amssymb}
\usepackage{graphicx}
\usepackage{color}

\usepackage{subfigure} 

\usepackage{algorithm}
\usepackage{algorithmic}
\usepackage{url}
\urldef{\mailsa}\path|{jun.wang, adam.woznica}@unige.ch|    
\urldef{\mailsb}\path|{alexandros.kalousis}@hesge.ch|    
\newcommand{\keywords}[1]{\par\addvspace\baselineskip
\noindent\keywordname\enspace\ignorespaces#1}

\begin{document}

\mainmatter  

\title{Learning Neighborhoods for Metric Learning}
\titlerunning{Learning Neighborhoods for Metric Learning}

%
%
\author{Jun Wang %
\and Adam Woznica\and Alexandros Kalousis}
\authorrunning{Wang et al.}

\institute{AI Lab, Department of Computer Science, University of Geneva, Switzerland\\
Department of Business Informatics, University of Applied Sciences, Western Switzerland
\mailsa\\
\mailsb\\
}

%
%

\toctitle{Lecture Notes in Computer Science}
\tocauthor{Authors' Instructions}
\maketitle

\begin{abstract}
Metric learning methods have been shown to perform well on different learning tasks. Many of 
them rely on target neighborhood relationships that are computed in the original feature space and 
remain fixed throughout learning.  As a result, the learned metric reflects the original neighborhood relations. 
We propose a novel formulation of the metric learning problem in which, in addition to the metric, 
the target neighborhood relations are also learned in a two-step iterative approach. The new formulation
can be seen as a generalization of many existing metric learning methods.  The formulation includes 
a target neighbor assignment rule that assigns different numbers of neighbors to instances according 
to their quality; `high quality' instances get more neighbors.
We experiment with two of its instantiations that correspond to the metric learning algorithms LMNN and MCML and 
compare it to other metric learning methods on a number of datasets. The experimental results show state-of-the-art 
performance and provide evidence that learning the neighborhood relations does improve predictive performance.
\end{abstract}

\section{Introduction}
The choice of the appropriate distance metric plays an important role in distance-based algorithms such as $k$-NN and $k$-Means clustering.
The Euclidean metric is often the metric of choice, however, it may easily decrease the performance of these algorithms 
since it relies on the simple assumption that all features are equally informative. Metric learning is an effective way to overcome 
this limitation by learning the importance of difference features exploiting prior knowledge that comes in different forms. 
The most well studied metric learning paradigm is that of learning the Mahalanobis metric with a steadily 
expanding literature over the last 
years~\cite{xing2003dml,schultz2004learning,globerson2006mlc,davis2007itm,nguyen2008metric,weinberger2009distance,lu2009geometry,guillaumin2009you,wang2011}.


Metric learning for classification relies on two interrelated concepts, similarity and dissimilarity constraints, and 
the target neighborhood. The latter defines for any given instance the instances that should be its neighbors and it 
is specified using similarity and dissimilarity constraints.
In the absence of any other prior knowledge the similarity and dissimilarity constraints are derived from 
the class labels; instances of the same class should be similar and instances of different classes should 
be dissimilar.

The target neighborhood can be constructed in a \textit{global} or \textit{local} manner. With a global target neighborhood all 
constraints over all instance pairs are active; {\em all} instances of the same class should be similar and {\em all} instances from different 
classes should be dissimilar~\cite{xing2003dml,globerson2006mlc}. These admittedly hard to achieve constraints can be relaxed with the incorporation of slack variables~\cite{schultz2004learning,davis2007itm,nguyen2008metric,lu2009geometry}.
With a local target neighborhood the satisfiability of the constraints is examined within a local 
neighborhood~\cite{goldberger2005nca,weinberger2006dml,nguyen2008metric,weinberger2009distance}.
For any given instance we only need to ensure that we satisfy the constraints that involve that instance and instances from its local neighborhood. 
The resulting problem is considerably less constrained than what we get with the global approach and easier to solve.  
However, the appropriate definition of the local target neighborhood becomes now a critical component of the metric learning algorithm since it 
determines which constraints will be considered in the learning process. \cite{weinberger2009distance} defines the local 
target neighborhood of an instance as its $k$, same-class, nearest neighbors, under the Euclidean 
metric in the original space. Goldberger et al.~\cite{goldberger2005nca} initialize the target neighborhood for each instance 
to all same-class instances. The local neighborhood is  encoded as a soft-max function of a linear projection matrix and changes 
as a result of the metric learning. 
%
%
%
%
%
%
With the exception of~\cite{goldberger2005nca}, all other approaches whether global or local establish a
target neighborhood prior to learning and keep it fixed throughout the learning process. Thus the metric 
that will be learned from these fixed neighborhood relations is constrained by them and will be a reflection 
of them. However, these relations are not necessarily optimal with respect to the learning problem that one is 
addressing. 

In this paper we propose a novel formulation of the metric learning problem that includes in the learning 
process the learning of the local target neighborhood relations.  The  formulation is based on the fact that 
many metric learning algorithms can be seen as directly maximizing the sum of some quality measure of the 
target neighbor relationships under an explicit parametrization of the target neighborhoods. We cast 
the process of learning the neighborhood as a linear programming problem with a totally unimodular constraint 
matrix~\cite{sierksma2002linear}. An integer 0-1 solution of the target neighbor relationship is guaranteed by the totally unimodular 
constraint matrix. The number of the target neighbors does not need to be fixed, the formulation allows the 
assignment of a different number of target neighbors for each learning instance according to the instance's quality.
We propose a two-step iterative optimization algorithm that learns the target neighborhood 
relationships and the distance metric.  The proposed neighborhood learning method can be coupled 
with standard metric learning methods to learn the distance metric, as long as these can be cast
as instances of our formulation.

We experiment with two instantiations of our approach where the  Large Margin Nearest Neighbor (LMNN)~\cite{weinberger2009distance}
 and Maximally Collapsing Metric Learning (MCML)~\cite{globerson2006mlc}
algorithms are used to learn the metric; we dub the respective instantiations LN-LMNN and LN-MCML. We 
performed a series of experiments on a number of classification problems in order to determine whether 
learning the neighborhood relations improves over only learning the distance metric. The experimental 
results show that this is indeed the case. In addition, we also compared our method with other state-of-the-art metric 
learning methods and show that it improves over the current state-of-the-art performance.

The paper is organized as follows.  In section~\ref{sec:relatedWork}, we discuss in more detail the related work. In Section~\ref{sec:LNML} we present the optimization problem of the 
Learning Neighborhoods for Metric Learning algorithm (LNML) and in Section~\ref{sec:optAlgorithm} we 
discuss the properties of LNML.  In Section~\ref{sec:ML} we instantiate our neighborhood learning method 
on LMNN and MCML. In Section~\ref{sec:experiments} we present the experimental results and we finally conclude with Section~\ref{sec:discussion}.
\section{Related Work}
\label{sec:relatedWork}
The early work of Xing et al., \cite{xing2003dml}, learns a Mahalanobis distance metric for clustering that tries to minimize 
the sum of pairwise distances between similar instances while keeping the sum of dissimilar instance distances greater than a 
threshold.  The similar and dissimilar pairs are determined on the basis of prior knowledge. 
Globerson \& Roweis, \cite{globerson2006mlc} introduced the Maximally Collapsing Metric Learning (MCML).
MCML uses a stochastic nearest neighbor selection rule which selects the nearest neighbor $\mathbf x_j$ of an instance $\mathbf x_i$
 under some probability distribution. It casts the metric learning problem as an optimization problem that tries to 
minimize the distance between two probability distributions, an ideal one and a data dependent one. In the ideal distribution the
selection probability  is always one for instances of the same class and zero for instances of different class, defining in that manner
the similarity and dissimilarity constraints under the global target neighborhood approach. In the data dependent distribution the selection
probability is given by a soft max function of a Mahalanobis distance metric, parametrized by the matrix $\mathbf M$ to be learned.
In a similar spirit Davis et al., ~\cite{davis2007itm}, introduced Information-Theoretic Metric Learning. ITML learns a Mahalanobis
metric for classification with similarities and dissimilarities constraints that follow the global target neighborhood approach. In ITML
all same-class instance pairs should have a distance smaller than some threshold and all different-class instance pairs should 
have a distance larger than some threshold. In addition the objective function of ITML seeks to minimize the distance between 
the learned metric matrix and a prior metric matrix, modelling like that prior knowledge about the metric if such is available. The optimization problem is cast as a distance of 
distributions subject to the pairwise constraints and finally expressed as a Bregman optimization problem (minimizing the LogDet divergence).
In order to be able to find a feasible solution they introduce slack variables in the similarity and dissimilarity constraints.


The so far discussed metric learning methods follow the global target neighborhood approach in which all instances of the same
class should be similar under the learned metric, and all pairs of instances from different classes dissimilar. This is a rather
hard constraint and assumes that there is a linear projection of the original feature space that results in unimodal class conditional
distributions. Goldberger et al., \cite{goldberger2005nca}, proposed the NCA metric learning method which uses the same stochastic 
nearest neighbor selection rule under the same data-dependent probability distribution as MCML. 
NCA seeks to minimize the soft error under its stochastic nearest neighbor selection rule.  It uses only similarity constraints 
and the original target neighborhood of an instance is the set of all same-class instances. After metric learning 
some, but not necessarily all, same class instances will end up having high probability of been selecting as 
nearest neighbors of a given instance, thus having a small distance, while the others will be pushed further away. 
NCA thus learns the local target neighborhood as a part of the optimization. 
Nevertheless it is prone to overfitting, \cite{yang2007regularized}, 
and does not scale to large datasets. The large margin nearest neighbor method (LMNN) described 
in~\cite{weinberger2006dml,weinberger2009distance} learns a distance 
metric which directly minimizes the distances of each instance to its local target neighbors while keeping a large margin between them and 
different class instances. The target neighbors have to be specified prior to metric learning and in the absence of prior knowledge these are
the $k$ same class nearest neighbors for each instance. 

\section{Learning Target Neighborhoods for Metric Learning}
\label{sec:LNML}
Given a set of training instances $\{(\mathbf x_1,y_1),(\mathbf x_2,y_2)$ $,\ldots,(\mathbf x_n,y_n)\}$
where $\mathbf x_i\in \mathbb {R}^d$ and the class labels $ y_i\in \{1,2,\ldots,c\}$, the Mahalanobis distance 
between two instances $\mathbf x_i$ and $\mathbf x_j$ is defined as: 
\begin{equation}
D_{\mathbf M}(\mathbf x_i,\mathbf x_j)=(\mathbf x_i-\mathbf x_j)^T{\mathbf M}(\mathbf x_i-\mathbf x_j) 
\end{equation}
where $\mathbf M$ is a Positive Semi-Definite (PSD) matrix ($\mathbf{M} \succeq 0$) that we will learn.

We can reformulate many of the existing metric learning methods, 
such as~\cite{xing2003dml,schultz2004learning,globerson2006mlc,nguyen2008metric,weinberger2009distance}, 
by explicitly parametrizing the target neighborhood relations as follows:
\begin{eqnarray}
\label{ML}
\min_{\mathbf { {M,\Xi}}}&&\sum_{ij,i \neq j,y_i=y_j}\mathbf P_{ij}\cdot F_{ij}({\mathbf {M,\Xi}}) \\\nonumber
s.t. && \mbox{  constraints of the original problem }\nonumber
\end{eqnarray}
The matrix $\mathbf P, \mathbf P_{ij} \in \{0,1\}$, describes the target neighbor relationships which are established prior to 
metric learning and are not altered in these methods. $\mathbf {P}_{ij}=1$, if 
$\mathbf x_{j}$ is the target neighbor of $\mathbf x_{i}$, otherwise, $\mathbf {P}_{ij}=0$. Note that
the parameters $\mathbf P_{ii}$ and $\mathbf P_{ij}: y_i \neq y_j$ are set to zero, since an instance $\mathbf x_i$ cannot be a target neighbor 
of itself and the target neighbor relationship is constrained to same-class instances.
This is why we have $i \neq j, y_i=y_j$ in the sum, however, for simplicity we will drop it from the following equations. 
$F_{ij}({\mathbf {M,\Xi}})$ is the term of the objective function of the metric learning 
methods that relates to the target neighbor relationship $\mathbf P_{ij}$, $\mathbf M$ is the Mahalanobis 
metric that we want to learn, and $\mathbf \Xi$ is a set of other parameters in the original 
metric learning problems, e.g. slack variables. Regularization terms on the $\mathbf {M}$ and $\mathbf \Xi$ 
parameters can also be added into Problem~\ref{ML}~\cite{schultz2004learning,nguyen2008metric}.

The $F_{ij}({\mathbf {M,\Xi}})$ term can be seen as the 'quality' of the target neighbor relationship 
$\mathbf P_{ij}$ under the distance metric $\mathbf M$; a low value indicates a high quality neighbor 
relationship $\mathbf P_{ij}$. The different metric learning methods learn the $\mathbf M$ matrix
that optimizes the sum of the quality terms based on the a priori established target neighbor 
relationships; however, there is no reason to believe that these target relationships are the most appropriate for learning.

To overcome the constraints imposed by the fixed target neighbors we propose the Learning the Neighborhood for Metric Learning method 
(LNML) in which, in addition to the metric matrix $\mathbf M$, we also learn the target neighborhood matrix $\mathbf P$. LNML has as 
objective function the one given in Problem~\ref{ML} which we now optimize also over the target neighborhood matrix $\mathbf P$.
We add some new constraints in Problem~\ref{ML} which control for the size of the
target neighborhoods. The new optimization problem is the following:
\begin{eqnarray}
\label{LNML}
{\min_{\mathbf{M,\Xi,P}}} &&\sum_{ij}\mathbf P_{ij}\cdot F_{ij}(\mathbf M, \Xi) \\ \nonumber
s.t.   
     && \sum_{i,j}{\mathbf {\mathbf P}}_{ij}=K_{av}*n \\\nonumber
     && {K_{max}} \geq \sum_{j}{\mathbf{\mathbf P}}_{i,j} \geq {K_{min}} \\\nonumber
     && 1 \geq {\mathbf {\mathbf P}}_{ij} \geq 0 \\ \nonumber
     && \mbox{constraints of the original problem}          \nonumber
\end{eqnarray} 
$K_{min}$ and $K_{max}$ are the minimum and maximum numbers of target neighbors that an instance can have.
Thus the second constraint controls the number of target neighbor that $\mathbf x_i$ instance can have.
$K_{av}$ is the average number of target neighbor per instance. It holds by construction that 
${K_{max}} \geq {K_{av}} \geq {K_{min}}$. We should note here that we relax the target neighborhood 
matrix so that its elements $\mathbf P_{ij}$ take values in $[0,1]$ (third constraint). However, we will show later 
that a solution $\mathbf P_{ij} \in \{0,1\}$ is obtained, given some natural constraints on the $K_{min}$, $K_{max}$ 
and $K_{av}$ parameters. 

\subsection{Target neighbor assignment rule}
Unlike other metric learning methods, e.g. LMNN, in which the number of target neighbors is fixed, LNML can
assign a different number of target neighbors for each instance. As we saw the first constraint
in Problem~\ref{LNML} sets the average number of target neighbors per instance to $K_{av}$, while the second constraint
limits the number of target neighbors for each instance between $K_{min}$ and $K_{max}$.
The above optimization problem implements a target neighbor
assignment rule which assigns more target neighbors to instances
that have high quality target neighbor relations.
We do so in order to avoid overfitting since most often the 'good' quality instances defined 
by metric learning algorithms~\cite{globerson2006mlc,weinberger2009distance} are instances in 
dense areas with low classification error.  As a result the geometry of the dense areas of the different classes will be emphasized. 
 How much emphasis we give on good quality instances depends on the actual values of 
${K_{min}}$  and ${K_{max}}$. 
 In the limit one can set the value of ${K_{min}}$ to zero; nevertheless the risk with such a strategy
is to focus heavily on dense and easy to learn regions of the data and ignore important boundary instances that 
are useful for learning.

\section{Optimization}
\label{sec:optAlgorithm}
\subsection{Properties of the Optimization Problem}
We will now show that we get integer solutions for the $\mathbf P$ matrix by solving a linear 
programming problem and analyze the properties of Problem~\ref{LNML}.
\begin{lemma}
\label{integer}
Given $\mathbf{{M, \Xi}}$, and ${K_{max}} \geq {K_{av}} \geq {K_{min}}$ then $\mathbf{{P}}_{ij} \in \{0,1\}$, 
if $K_{min}$, $K_{max}$ and $K_{av}$ are integers.
\end{lemma}
\begin{proof}
Given $\mathbf{{M}}$ and $\mathbf{{\Xi}}$, $F_{ij}(\mathbf {M,\Xi})$ becomes a constant. We 
denote by 
${\mathbf p}$
the vectorization of the target neighborhood matrix $\mathbf P$ which excludes the diagonal elements and $\mathbf P_{ij}:y_i \neq y_j$, and by $\mathbf f$ 
the respective vectorized version of the $F_{ij}$ terms. Then we rewrite Problem~\ref{LNML} as:
\begin{eqnarray}
\label{integer-opt}
     \min_{{\mathbf p}} & & {\mathbf p}^T{\mathbf f} \nonumber \\ 
s.t.  & &(\underbrace{K_{max},\cdots,K_{max}}_n, K_{av}*n)^T \geq \mathbf A \mathbf p \geq\nonumber\\  &&{(\underbrace{K_{min},\cdots,K_{min}}_n,K_{av}*n)^T}\nonumber\\
      & & 1 \geq \mathbf p_i \geq 0 
\end{eqnarray}
The first and second constraints of Problem~\ref{LNML} are reformulated as the first constraint in 
Problem~\ref{integer-opt}. $\mathbf A$ is a $(n+1) \times (\sum_{c_l}n^2_{c_l}-n)$ constraint matrix, where $n_{c_l}$ is the number of instances in class $c_l$
\[
{\mathbf{A}}=\left[
\begin{array}{cccc}
\mathbf 1            & \mathbf 0&\cdots &\mathbf 0\\
\mathbf 0            & \mathbf 1&\cdots &\mathbf 0\\
\vdots  & \vdots  & \ddots & \vdots  \\
\mathbf 0            & \mathbf 0&\cdots &\mathbf 1\\
\mathbf 1            & \mathbf 1&\cdots &\mathbf 1
\end{array}
\right]
\]
where $\mathbf 1$ ($\mathbf 0$)  is the vector of ones (zeros). Its elements depends on the its position in the matrix $\mathbf A$. In its $i$th column, all $\mathbf 1$ ($\mathbf 0$) vectors have $n_i-1$ elements, where $n_i$ is the number of instances of class $c_j$ with $c_j=y_{p_i}$.
According to the sufficient condition 
for total unimodularity (Theorem 7.3 in~\cite{sierksma2002linear}) the constraint matrix 
$\mathbf A$ is a totally unimodular matrix. Thus, the constraint matrix 
$\mathbf B=[\mathbf{I, -I, A, -A}]^T$ in the following equivalent problem also is a totally unimodular matrix (pp.268 in~\cite{schrijver1998theory}). 
\begin{eqnarray}
\label{integer-B}
\min_{\mathbf p} && \mathbf p^T \mathbf f \nonumber\\ 
s.t.&&\mathbf B {\mathbf p}  \leq {\mathbf e}\nonumber\\                           
&&e=(\underbrace{1,\cdots,1}_{\sum_{c_l}n^2_{c_l}-n},\underbrace{0,\cdots,0}_{\sum_{c_l}n^2_{c_l}-n},\underbrace{K_{max},\cdots,K_{max}}_n,\nonumber\\  
&&K_{av}*n,\underbrace{-K_{min},\cdots,-K_{min}}_n,-K_{av}*n)^T   
\end{eqnarray}
Since $\mathbf e$ is an integer vector, provided $K_{min}$, $K_{max}$, and $K_{av}$, are integers, and the constraint matrix $\mathbf B$ 
is totally unimodular, the above linear programming problem will only have integer solutions (Theorem 19.1a in~\cite{schrijver1998theory}).
Therefore, for the solution $\mathbf p$ it will hold that $\mathbf p_i \in \{0,1\}$ and consequently $\mathbf P_{ij} \in \{0,1\}$.
\end{proof}

Although the constraints to control the size of the target neighborhood are convex, the objective function 
in Problem~\ref{LNML} is not jointly convex in $\mathbf P$ and $(\mathbf M,\mathbf \Xi)$. However, as 
shown in Lemma~\ref{integer}, the binary solution of $\mathbf P$ can be obtained by a simple linear program 
if we fix $(\mathbf{M,\Xi})$. Thus, Problem~\ref{LNML} is  individually convex in $\mathbf P$ and $(\mathbf M,\mathbf \Xi)$, 
if the original metric learning method is convex; this condition holds for all the methods that 
can be coupled with our neighborhood learning method~\cite{xing2003dml,schultz2004learning,globerson2006mlc,nguyen2008metric,weinberger2009distance}.

\subsection{Optimization Algorithm}
Based on Lemma~\ref{integer} and the individual convexity property we propose a general and easy to implement 
iterative algorithm to solve Problem~\ref{LNML}. The details are given in Algorithm~\ref{algo:LNML}.
At the first step of the $k$th iteration we learn the binary target neighborhood matrix $\mathbf P^{(k)}$ under a fixed  
metric matrix $\mathbf M^{(k-1)}$ and $\mathbf \Xi^{(k-1)}$, learned in the $k-1$th iteration, by solving the linear programming problem described 
in Lemma~\ref{integer}. At the second step of the iteration we learn the metric matrix $\mathbf M^{(k)}$ and $\mathbf{{\Xi}}^{(k)}$ with the target 
neighborhood matrix $\mathbf P^{(k)}$ using as the initial metric matrix the $\mathbf M^{(k-1)}$. 
The second step is simply the application of a standard metric learning algorithm in which we set the target neighborhood matrix to 
the learned $\mathbf P^{(k)}$ and the initial metric matrix to $\mathbf M^{(k-1)}$. The convergence of proposed algorithm is guaranteed 
if the original metric learning problem is convex~\cite{bezdek2002some}. In our experiment, it most often converges in 5-10 iterations.
\begin{algorithm}[tb]
   \caption{LNML}
   \label{algo:LNML}
\begin{algorithmic}
   \STATE {\bfseries Input:} $\mathbf{{X}}$, $\mathbf{{Y}}$, $\mathbf{{M}}^0$,$\mathbf{{\Xi}}^0$, $K_{min}$, $K_{max}$ and $K_{av}$ 
   \STATE {\bfseries Output:} $\mathbf{{M}}$ 
   \REPEAT
   \STATE $\mathbf{{P}}^{(k)}$=LearningNeighborhood($\mathbf X,\mathbf Y,\mathbf{{M}}^{(k-1)},\mathbf{{\Xi}}^{(k-1)}$) by solving Problem ~\ref{integer-opt} 
   \STATE $(\mathbf{{M}}^{(k)},\mathbf{\Xi}^{(k)})$=MetricLearning($\mathbf{{M}}^{(k-1)}$,$\mathbf{P}^{(k)}$)
   \STATE $k:=k+1$
   \UNTIL{convergence}
\end{algorithmic}
\end{algorithm}

\section{Instantiating LNML}
\label{sec:ML}
In this section we will show how we instantiate our neighborhood learning method with two standard metric learning methods, LMNN and MCML,
other possible instantiations include the metric learning methods presented in~\cite{xing2003dml,schultz2004learning,nguyen2008metric}.

\subsection{Learning the Neighborhood for LMNN}
 

The optimization problem of LMNN is given by:
\begin{eqnarray}
\label{LMNN-opt-prob}
\min_{\mathbf{ M, \xi}} &&\sum_{ij} \mathbf {P}_{ij}\{(1-\mu) D_{\mathbf {M}}(\mathbf x_i,\mathbf x_j)
+\mu \sum_{l}(1-\mathbf{Y}_{il})\xi_{ijl}\} \\ \nonumber
s.t. & & D_{\mathbf M}(\mathbf x_i,\mathbf x_l)-D_{\mathbf M}(\mathbf x_i,\mathbf x_j) \geq 1-\xi_{ijl}\\ \nonumber
     & & \xi_{ijl}>0\\ \nonumber
     & & \mathbf M \succeq 0
\end{eqnarray}
where the matrix $\mathbf Y, \mathbf Y_{ij} \in \{0,1\},$ indicates whether the class labels 
$y_i$ and $y_j$ are the same ($\mathbf Y_{ij}=1$) or different ($\mathbf Y_{ij}=0$).
The objective is to minimize the sum of the distances of all instances to their target 
neighbors while allowing for some errors, this trade off is controlled by the $\mu$ parameter. 
This is a convex optimization problem that has been shown to have good generalization ability
and can be applied to large datasets. The original problem formulation
corresponds to a fixed parametrization of $\mathbf P$ where its non-zero values are given by the 
$k$ nearest neighbors of the same class.


Coupling the neighborhood learning framework with the LMNN metric learning method results in the following optimization problem:
\begin{eqnarray}
\label{LNLMNN}
\min_{\mathbf{M,P,\xi}}&&\sum_{ij}\mathbf P_{ij}\cdot F_{ij}(\mathbf M, \xi) \\ \nonumber
=\min_{\mathbf M, \mathbf P,\xi} & &  \sum_{ij}{\mathbf {P}}_{ij}\{(1-\mu) D_{\mathbf {M}}(\mathbf x_i,\mathbf x_j)
+\mu \sum_{l}(1-\mathbf{Y}_{il})\xi_{ijl}\}\\ \nonumber
s.t.    & & {K_{max}} \geq \sum_{j}{\mathbf{P}}_{i,j} \geq {K_{min}} \\ \nonumber
       & & \sum_{i,j}{\mathbf {P}}_{ij}=K_{av}*n\\ \nonumber
       & & 1 \geq {\mathbf {P}}_{ij} \geq 0  \\ \nonumber
       & & D_{\mathbf {M}}(\mathbf x_i,\mathbf x_l)-D_{\mathbf {M}}(\mathbf x_i,\mathbf x_j) \geq 1-\xi_{ijl}\\ \nonumber
       & & \xi_{ijl}>0\\ \nonumber
       & &  {\mathbf {M}} \succeq 0 \nonumber
\end{eqnarray}
We will call this coupling of LNML and LMNN LN-LMNN. 
The target neighbor assignment rule of LN-LMNN assigns more target neighbors to instances that have small distances 
from their target neighbors and low hinge loss. 

\subsection{Learning the Neighborhood for MCML}
MCML relies on a data dependent stochastic probability that an instance $\mathbf x_j$ is selected as the 
nearest neighbor of an instance $\mathbf x_i$; this probability is given by:
\begin{eqnarray}
\label{stochastic-rule}
p_\mathbf{M}(j|i)=\frac{e^{-D_\mathbf {M}(\mathbf x_i,\mathbf x_j)}}{ Z_i}=
\frac{e^{-D_\mathbf M(\mathbf x_i,\mathbf x_j)}}{\sum_{k \ne i}{e^{-D_\mathbf M(\mathbf x_i,\mathbf x_k)}}}, & i \ne j\nonumber\\
\end{eqnarray}
MCML learns the Mahalanobis metric that minimizes the KL divergence distance between this 
probability distribution and the ideal probability distribution $p_0$ given by:
\begin{eqnarray}
\label{ideal distribution}
 p_0(j|i)=\frac{\mathbf P_{ij}}{\sum_{k}{\mathbf P_{ik}}}, & p_0(i|i)=0
 \end{eqnarray}
where $\mathbf P_{ij}=1$, if instance $\mathbf x_j$ is the target neighbor of instance $\mathbf x_i$, otherwise, $\mathbf P_{ij}=0$. The 
optimization problem of MCML is given by:
\begin{eqnarray}
\label{MCML}
\min_{\mathbf M}&& \sum_{i}KL[p_0(j|i)|p_\mathbf M(j|i)] \\ \nonumber
=\min_{\mathbf M}&& \sum_{i,j}{\mathbf P_{ij}} \frac{(D_\mathbf M(\mathbf x_i,\mathbf x_j)+\log  Z_i)}{\sum_{k}{\mathbf P_{ik}}}  \\ \nonumber
s.t.&&\mathbf M \succeq 0 \nonumber
\end{eqnarray}
Like LMNN, this is also a convex optimization problem. In the original problem formulation the ideal distribution 
is defined based on class labels, i.e. $\mathbf P_{ij}=1$, if instances $\mathbf x_i$ and $\mathbf x_j$ share the same 
class label, otherwise, $\mathbf P_{ij}=0$. 

The neighborhood learning method cannot learn directly the target neighborhood for MCML, since the objective function of the 
latter cannot be rewritten in the form of the objective function in Problem~\ref{LNML}, due to the 
denominator $\sum_{k}{\mathbf{P}}_{ik}$. However, if we fix the size of the neighborhood to 
$\sum_{k}{\mathbf{P}}_{i,k}=K_{av}=K_{min}=K_{max}$ the two methods can be coupled and 
the resulting optimization is given by: 
\begin{eqnarray}
\label{LNMCML}
\min_{\mathbf{M,P}}&&{\sum_{ij}\mathbf P_{ij}\cdot F_{ij}(\mathbf M)} \\ \nonumber
=\min_{\mathbf {M,P}} &&\sum_{i,j}{\mathbf P_{ij}}\frac{(D_\mathbf M(\mathbf x_i, \mathbf x_j)+\log Z_i)}{K_{av}}\\ \nonumber
s.t. && \sum_{j}{\mathbf{\mathbf P}}_{i,j} = {K_{av}} \\ \nonumber
&& \mathbf M \succeq 0
\end{eqnarray}
We will dub this coupling of LNML and MCML as LN-MCML. 
The original MCML method follows the global approach in establishing the neighborhood, with LN-MCML
we get a local approach in which the neighborhoods are of fixed size $K_{av}$ for every instance. 

\section{Experiments}
\label{sec:experiments}
With the experiments we wish to investigate a number of issues. First, we want to examine whether learning the target 
neighborhood relations in the metric learning process can improve predictive performance over the baseline approach of metric 
learning with an apriori established target neighborhood. 
Second, we want to acquire an initial understanding of how the parameters $K_{min}$ and $K_{max}$ relate to the predictive performance. To this end, we will examine the
predictive performance of LN-LMNN with two fold inner Cross Validation (CV) to select the appropriate values of $K_{min}$ and $K_{max}$, method which we will denote by 
LN-LMNN(CV), and that of LN-LMNN, with a default setting of $K_{min}=K_{max}=K_{av}$.  Finally, we want to see how the method that we propose 
compares to other state of the art metric learning methods, namely NCA and ITML. We include as an additional  
baseline in our experiments the performance of the Euclidean metric (EucMetric).  We experimented with twelve different 
datasets: seven from the UCI machine learning repository, Sonar, Ionosphere, Iris, Balance, Wine, Letter, Isolet; four text mining datasets, Function, Alt, 
Disease and Structure, which were constructed from biological corpora~\cite{kalousis2007stability}; and MNIST~\cite{MNIST}, a handwritten 
digit recognition problem. A more detailed description of the datasets is given in Table~\ref{datasets}. 

Since LMNN is computationally expensive for datasets with large number of features we applied principal component analysis (PCA)
to retain a limited number of principal components,  following~\cite{weinberger2009distance}. The datasets to which this was done were 
the four text mining datasets, Isolet and MNIST. For the two latter 173 and 164 principal components were respectively retained  
that explain 95\% of the total variance. For the text mining datasets more than 1300 principal components should be retained 
to explain 95\% of the total variance. Considering the running time constraints, we kept the 300 most important principal components 
which explained 52.45\%, 47.57\%, 44.30\% and 48.16\% of the total variance for respectively Alt, Disease, Function and Structure.  
We could experiment with NCA and MCML on full tranining datasets only with datasets with a small number of instances due to their computational complexity. For 
completeness we experimented with NCA on large datasets by undersampling the training instances, i.e. the learning process only involved 10\% of full training 
instances which was the maximum number we could experiment for each dataset. We also applied ITML on both versions of the larger datasets, i.e. with PCA-based 
dimensionality reduction and the original ones.

\begin{table*}[bt]
\begin{center}
\caption{Datasets.}
\label{datasets}
\vskip 0.15in
 \scalebox{0.7}{
\begin{tabular}{l|l|cccccc} 
  Datasets   &  Description                     & \# Sample& \# Feature &  \# Class &\# Retained PCA Components& \% Explained Variance             \\ \hline
  Sonar      &                                  & 208     & 60        & 2        & NA                 &  NA                          \\
  Ionosphere &                                  & 351     & 34        & 2        & NA                 &  NA                          \\
  Wine       &                                  & 178     & 13        & 3        & NA                 &  NA                          \\
  Iris       &                                  & 150     & 4         & 3        & NA                 &  NA                          \\
  Balance    &                                  & 625     & 4         & 3        & NA                 &  NA                          \\ \hline
  Letter     & character recognition            & 20000   & 16        & 26       & NA                 &  NA                          \\
  Function   & sentence classification          & 3907    & 2708      & 2        & 300                &  44.30\%                     \\
  Alt        & sentence classification          & 4157    & 2112      & 2        & 300                &  52.45\%                     \\
  Disease    & sentence classification          & 3273    & 2376      & 2        & 300                &  47.57\%                     \\
  Structure  & sentence classification          & 3584    & 2368      & 2        & 300                &  48.16\%                     \\
  Isolet     & spoken character recognition     & 7797    & 619       & 26       & 173                &  95\%                                 \\
  MNIST      & handwritten digit recognition   & 70000   & 784       & 26       & 164                &  95\%                                  \\ 
\end{tabular}
}
\end{center}
\vskip -0.1in
\end{table*}

For ITML, we randomly generate for each dataset the default $20c^2$ constraints which are bounded repectively by the 5th and 95th percentiles of the
distribution of all available pairwise distances for similar and dissimilar pairs. The slack variable $\gamma$ is chosen form $\{10^i\}^4_{i=-4}$ 
using two-fold CV. The default 
identity matrix is employed as the regularization matrix. 
For the different instantiations of the LNML method we took care to have the same parameter settings for the encapsulated metric learning 
method and the respective baseline metric learning. For LN-LMNN, LN-LMNN(CV) and LMNN the regularization parameter $\mu$ that controls the trade-off 
between the distance minimization component and the hinge loss component was set to 0.5 (the default value of LMNN). For LMNN the default 
number of target neighbors was used (three). For LN-LMNN, we set $K_{min}=K_{max}=K_{av}=3$, similar to LMNN. 
To explore the effect of a flexible neighborhood, the values of the $K_{min}$ and $K_{max}$ parameters in LN-LMNN(CV) were 
selected from the sets $\{1,4,3\}$ and $\{2,5,3\}$ respectively, while $K_{av}$ was fixed to three. Similarly for LN-MCML we also set $K_{av}=3$.
The distance metrics for all methods are initialized to the Euclidean metric. As the classification algorithm we used 1-Nearest 
Neighbor.  

We used 10-fold cross validation for all datasets to estimate classification accuracy, with the exception of Isolet and MNIST 
for which the default train and test split was used. The statistical significance of the differences were tested with McNemar's test and the p-value was set to 0.05. 
In order to get a better understanding of the relative performance of the different algorithms for a given dataset we used a ranking schema in which an algorithm  
A was assigned one point if it was found to have a significantly better accuracy than another algorithm B, 0.5 points if the two algorithms did not have a 
significantly different performance, and zero points if A was found to be significantly worse than B. The rank of an algorithm for a given dataset is simply 
the sum of the points over the different pairwise comparisons. When comparing $N$ algorithms in a single dataset the highest possible score is $N-1$ 
while if there is no significant difference each algorithm will get $(N-1)/2$ points.

\subsection{Results}
The results are presented in Table \ref{results}. Examining whether learning also the neighborhood improves the predictive performance compared to plain metric
learning, we see that in the case of LN-MCML, and for the five small datasets for which we have results, learning 
the neighborhood results in a statistically significant deterioration of the accuracy in one out of the five datasets 
(balance), while for the remaining four the differences were not statistically significant. If we now examine LN-LMNN(CV), LN-LMNN  
and LMNN we see that here learning the neighborhood does bring a statistically significant improvement. More precisely, LN-LMNN(CV) and LN-LMNN improve over LMNN respectively in six (two small and four large) and four (two small and two large)
out of the 12 datasets. Moreove, by comparing LN-LMNN(CV) and LN-LMNN, we see that 
learning a flexible neighborhood with LN-LMNN(CV) improves significantly the performance over LN-LMNN on two datasets. The low performance of LN-MCML on the 
balance dataset was intriguing; in order to take a closer look we tried to determine automatically the appropriate target neighborhood size, $K_{av}$,
by selecting it on the basis of five-fold inner cross validation from the set $K_{av}= \{3,5,7,10,20,30\}$. The results 
showed that the default value of $K_{av}$ was too small for the given dataset, with the average selected size of the target 
neighborhood at 29. As a result of the automatic tunning of the target neighborhood size the predictive performance of LN-MCML
jumped at an accuracy of 93.92\% which represented a significant improvement over the baseline MCML for the balance dataset.
For the remaining datasets it turned out that the choice of $K_{av}=3$ was a good default choice.  In any case, determining 
the appropriate size of the target neighborhood and how that affects the predictive performance is an issue that we wish to investigate further. 
In terms of the total score that the different methods obtain the LN-LMNN(CV) achieves the best in both the small and large datasets. 
It is followed closely by NCA in the small datasets and by LN-LMNN in the large datasets.


\begin{table*}[bt]
\begin{center}

\caption{
Accuracy results. The superscripts $^{+-=}$ next to the LN-XXXX accuracy indicate the result of the McNemar's
statistical test result of its comparison to the accuracy of XXXX and denote respectively a significant win, loss or no difference 
for LN-XXXX. Similarly, the superscripts $^{+-=}$ next to the LN-LMNN(CV) accuracy indicate the result of its comparison to the accuracies of LMNN and LN-LMNN. The bold entries for each dataset have no significant difference from the best accuracy for that dataset. The number in 
the parenthesis indicates the score of the respective algorithm for the given dataset based on the pairwise comparisons.
}
\label{results}
\subtable[Small datasets]{
\centering
\label{smalldata}
 \scalebox{0.7}{
\begin{tabular}{l|cc|ccc|c|c|c} 
  Datasets        &MCML              &LN-MCML   & LMNN       & LN-LMNN   & LN-LMNN(CV)  &EucMetric             &NCA                  &ITML                \\ \hline
Sonar & \textbf{ 82.69}(3.5)& \textbf{ 84.62}(3.5)$^=$& \textbf{ 81.25}(3.5)& \textbf{ 81.25}(3.5)$^=$& \textbf{ 83.17}(3.5)$^{==}$& \textbf{ 80.77}(3.5)& \textbf{ 81.73}(3.5)& \textbf{ 82.69}(3.5)\\
Ionosphere &  88.03 (3.0)& \textbf{ 88.89}(3.5)$^=$& \textbf{ 89.17}(3.5)&  87.75 (3.0)$^=$& \textbf{ 92.02}(5.5)$^{=+}$&  86.32 (3.0)& \textbf{ 88.60}(3.5)&  87.75 (3.0)\\
Wine &  91.57 (3.0)& \textbf{ 96.07}(4.0)$^=$&  94.38 (3.0)& \textbf{ 97.75}(5.5)$^+$& \textbf{ 97.75}(5.5)$^{+=}$&  76.97 (0.0)&  91.57 (3.0)& \textbf{ 94.94}(4.0)\\
Iris & \textbf{ 98.00}(4.5)& \textbf{ 96.00}(3.5)$^=$& \textbf{ 96.00}(3.5)&  94.00 (3.0)$^=$&  94.00 (3.0)$^{==}$& \textbf{ 96.00}(3.5)& \textbf{ 96.00}(3.5)& \textbf{ 96.00}(3.5)\\
Balance &  91.20 (5.0)&  78.08 (1.0)$^-$&  78.56 (1.0)&  89.12 (4.5)$^+$&  89.28 (4.5)$^{+=}$&  78.72 (1.0)& \textbf{ 96.32}(7.0)&  87.84 (4.0)\\\hline
Total Score& 19.0 & 15.5 & 14.5 & 19.5 & 22.0 & 11.0 & 20.5 & 18.0 \\
\end{tabular}}
}
\subtable[Large datasets]{
\centering
\label{largedata}
 \scalebox{0.6}{
\begin{tabular}{l|ccc|cc|c|cc} 
  Datasets  & PCA+LMNN     & {PCA+LN-LMNN} & {PCA+LN-LMNN(CV)}   &EucMetric      &PCA+EucMetric  &PCA+NCA &ITML   &PCA+ITML                   \\ \hline
  
  Letter &  96.86 (5.0)& \textbf{ 97.71}(6.5)$^+$& \textbf{ 97.64}(6.5)$^{+=}$&  96.02 (0.5)&  96.02 (0.5)&  96.48 (3.0)&  96.39 (3.0)&  96.39 (3.0)\\
  Function &  76.30 (2.5)&  76.73 (2.5)$^=$& \textbf{ 78.91}(6.0)$^{++}$& \textbf{ 78.73}(6.0)&  76.48 (2.5)&  72.36 (0.0)& \textbf{ 78.73}(6.0)&  76.45 (2.5)\\
Alt &  83.98 (5.0)& \textbf{ 84.92}(6.5)$^+$& \textbf{ 85.37}(6.5)$^{+=}$&  68.51 (0.5)&  71.33 (2.0)&  78.54 (4.0)&  68.49 (0.5)&  72.53 (3.0)\\
Disease & \textbf{ 80.23}(4.0)& \textbf{ 80.14}(4.0)$^=$& \textbf{ 80.66}(4.0)$^{==}$& \textbf{ 80.60}(4.0)& \textbf{ 80.23}(4.0)&  73.59 (0.0)& \textbf{ 80.60}(4.0)& \textbf{ 80.14}(4.0)\\
Structure &  77.87 (4.5)& \textbf{ 78.83}(6.0)$^=$& \textbf{ 79.37}(6.5)$^{+=}$&  75.82 (1.5)&  77.00 (4.0)&  71.93 (0.0)&  75.79 (1.5)&  77.06 (4.0)\\ 
Isolet & \textbf{ 95.96}(6.0)& \textbf{ 95.06}(6.0)$^=$& \textbf{ 95.06}(6.0)$^{==}$&  88.58 (1.5)&  88.33 (1.5)&   85.63(0.0)&  92.05 (3.5)&  91.08 (3.5)\\
   MNIST & \textbf{ 97.66}(6.0)& \textbf{ 97.66}(6.0)$^=$& \textbf{ 97.73}(6.0)$^{==}$&  96.91 (2.0)&  96.97 (2.0)&  96.58 (1.5)&  96.93 (1.5)&  97.09 (3.0)\\\hline
  Total Score      &33&37.5&41.5&16 &16.5&8.5&20&23 \\ 
\end{tabular}}
}
\end{center}
\vskip -0.1in
\end{table*}

\section{Conclusion and Future Work}
\label{sec:discussion}
We presented LNML, a general Learning Neighborhood method for Metric Learning algorithms which couples the metric learning 
process with the process of establishing the appropriate target neighborhood for each instance, i.e. discovering for each 
instance which same class instances should be its neighbors.  With the exception of NCA, which cannot be applied on 
datasets with many instances, all other metric learning methods whether they establish a global or a local target neighborhood 
do that prior to the metric learning and keep the target neighborhood fixed throughout the learning process. The metric that is 
learned as a result of the fixed neighborhoods simply reflects these original relations which are not necessarily optimal 
with respect to the classification problem that one is trying to solve. LNML lifts these constraints by learning the target 
neighborhood. We demonstrated it with two metric learning methods, LMNN and MCML. The experimental results show that learning 
the neighborhood can indeed improve the predictive performance. 

The target neighborhood matrix $\mathbf P$ is strongly related to the similarity graphs which are often used  in semi-supervised 
learning~\cite{JebaraICML2009}, spectral clustering \cite{Luxburg2007} and manifold learning~\cite{RoweisSaulLLE2000}. Most often 
the similarity graphs in these methods are constructed in the original space, which nevertheless can be quite different from true 
manifold on which the data lies. These methods could also profit if one is able to learn the similarity graph instead of basing 
it on some prior structure. 

\section*{Acknowledgments} 
This work was funded by the Swiss NSF (Grant 200021-122283/1). The support of the European
Commission through EU projects DebugIT (FP7-217139) and e-LICO (FP7-231519) is also gratefully
acknowledged.

\bibliography{LNML}
\bibliographystyle{plain}
\end{document}